\documentclass{article}

\usepackage{amsthm, amsmath}
\usepackage{amssymb}
\usepackage{xr}
\usepackage{bm}

\usepackage[pdftex]{graphicx}
\DeclareGraphicsExtensions{.pdf,.jpeg,.png}

\usepackage{colordvi}

\usepackage[round]{natbib}

\newtheorem{thm}{Theorem}
\newtheorem{prop}{Proposition}
\newtheorem{cor}{Corollary}

\newtheorem{defn}{Definition}

\theoremstyle{plain}

\newcommand{\reals}{\mathbb{R}}
\newcommand{\nb}{{\bm{n}}}
\newcommand{\eps}{\epsilon}
\newcommand{\sF}{{\mathcal F}}
\newcommand{\sX}{{\mathcal X}}
\newcommand{\pfhat}{\widehat{\alpha}}
\newcommand{\pdhat}{\widehat{p}}
\newcommand{\Rhat}{\widehat{R}}
\newcommand{\RMhatul}{\widehat{\underline{R}}_M}
\newcommand{\RMul}{\underline{R}_M}
\newcommand{\fhat}{\widehat{f}}
\newcommand{\pihat}{\widehat{\pi}}
\newcommand{\nuhat}{\widehat{\nu}}

\newcommand{\conv}{\operatornamewithlimits{conv}}
\newcommand{\supp}{\operatornamewithlimits{supp}}

\newcommand{\Ptest}{P_0}
\newcommand{\Ptesthat}{\widehat{P}_0}
\newcommand{\Phat}{\widehat{P}}
\newcommand{\ntest}{n_0}
\newcommand{\ind}[1]{{\bf 1}_{\{#1\}}}

\usepackage[margin=1.2in]{geometry}
\date{}

\title{\bf{Class Proportion Estimation with Application to Multiclass Anomaly Rejection}}

\author{Tyler Sanderson\thanks{Current affiliation: Google Inc} \ and Clayton Scott
\\\\ Electrical Engineering and Computer Science Department
\\University of Michigan, Ann Arbor, USA} 

\begin{document}
\maketitle

\begin{abstract}
This work addresses two classification problems that fall under the
heading of domain adaptation, wherein the distributions of training and testing
examples differ. The first problem studied is that of class proportion
estimation, which is the problem of estimating the class proportions in an
unlabeled testing data set given labeled examples of each class. Compared
to previous work on this problem, our approach has the novel feature that it
does not require labeled training data from one of the classes. This property
allows us to address the second domain adaptation problem, namely,
multiclass anomaly rejection. Here, the goal is to design a classifier
that has the option of assigning a ``reject" label, indicating that the
instance did not arise from a class present in the training data. We
establish consistent learning strategies for both of these domain
adaptation problems, which to our knowledge are the first of their kind.
We also implement the class proportion estimation technique and
demonstrate its performance on several benchmark data sets.
\end{abstract}

\section{Introduction}

This work studies two related classification problems that fall under the
heading of domain adaptation, which is used to describe any learning
problem where the distributions of training and testing instances differ.
In particular, we study the problems of class proportion estimation (CPE)
and multiclass anomaly rejection (MCAR). Both problems are studied in a
multiclass setting, where the learner has access to a labeled training
data set as well as an unlabeled testing data set.  CPE is the problem of
estimating the class proportions governing the unlabeled testing data,
which may differ from those in the training data set. Unlike
previous approaches to CPE, our approach has the novel feature that it
does not require training data from one of the classes. This property
allows us to address MCAR, where the goal is to design a classifier that
may assign a ``reject" label, indicating that the instance did not arise
from a class present in the training data. We establish consistent
learning strategies for both of these domain adaptation problems, which to
our knowledge are the first of their kind. We also implement the CPE
technique and demonstrate its performance on several benchmark data sets.

To begin, let us state the CPE problem. There are $M$
classes, and a training sample for each class:
\begin{equation}
\label{eqn:traindata}
X_1^i, \ldots, X_{n_i}^i \stackrel{iid}{\sim} P_i,
\end{equation}
where $P_i$ is the $i$th class-conditional distribution, and
$X_j^i$ denotes the $j$th training sample from class $i$. In
addition, there is an unlabeled testing sample
\begin{equation}
\label{eqn:testdata}
X_1^0, \ldots, X_{n_0}^0 \sim P_0 := \sum_{i=1}^M \pi_i P_i,
\end{equation}
drawn from a mixture of the different classes. Here $\pi_i \ge 0$ and
$\sum_i \pi_i = 1$. The critical feature of this
problem is that the proportions $\pi_i$ are unknown and
different from the proportions represented in the training
data, so that $n_i/\sum_\ell n_\ell$ is not a reasonable
estimate. The goal is to estimate the $\pi_i$ accurately, while
making minimal assumptions on the $P_i$.

This form of domain adaptation arises frequently in
applications where training and testing data are gathered according to
different sampling plans. For example, training data gathered
prospectively may have user-determined sample sizes, while testing data
analyzed retrospectively have sample sizes that are beyond the user's
control.

One motivation for class proportion estimation is design of a
classifier for the test distribution. Suppose that there is a joint
distribution on labels and instances with $P_0$ the marginal distribution
on instances, $P_i$ the
class-conditional distributions, and $\pi_i$ the prior distribution
on
labels. The risk of a classifier $f : \sX \to
\{1, \ldots, M\}$, $\sX \subseteq \reals^d$ denoting the feature space,
may
be expressed
$
R(f) := \sum_i \pi_i R_i(f)
$
where $R_i(f) := P_i(\{x : f(x) \ne i\})$. The
class-conditional errors $R_i$ can
be estimated since the training data provide examples from each class.
However, the class proportions $\pi_i$ need to be estimated in order to
estimate the risk and thereby achieve good generalization.\footnote{Note 
that there are two possible settings for evaluation. In a
transductive setting, the goal is to assign labels to the given
test examples, while in a semi-supervised setting, the goal
is to use these unlabeled examples to design a
general-purpose classifier for classifying future draws from
$P_0$. We focus on the semi-supervised setting, which can be
specialized to the transductive setting.}

Our work is further motivated by MCAR, another domain adaptation
problem. In particular, we consider the problem of having no training data
from the last class ($n_M = 0$), which we consider to be the anomaly
class. Many real problems fall into this category. For example, a
classifier for object recognition will undoubtedly encounter object types
in the real world not observed during training. The first $M-1$ classes
may be viewed as the known training classes, and predicting the $M$th
class amounts to a decision to ``reject" an instance as not belonging to
any of the known classes. This problem is more challenging than regular
multiclass classification because estimation of $R_M(f)$ is no longer
straightforward.

To summarize, this work makes the following contributions: It establishes 
the first methodology for CPE that is consistent in the case where a class 
is not observed. The first known consistent discrimination rule for MCAR 
is also introduced. Finally, we propose a practical implementation of our 
CPE methodology, and support this approach with experimental comparisons 
to existing methods.


On the technical side, our approach hinges on a reduction of CPE to
another problem called {\em mixture proportion estimation}, reviewed
below. To convert methods for CPE to a discrimination rule for MCAR, we
also introduce a novel error estimation strategy for use with empirical
risk minimization, and a corresponding uniform error analysis using
multiclass VC theory.

\section{Related Work}

Class proportion estimation goes back at least to \citet{hall81},
who introduced an approach for univariate data based on matching a
weighted combination of class-conditional empirical distribution functions
to the empirical distribution function of the unlabeled data. This idea
was extended by \citet{titterington83}, who replaced empirical
distribution functions by kernel density estimates, which allowed this
``distribution matching" method to extend easily to multivariate data. The
matching criterion is the $L^2$ distance between estimates of the marginal
density $P_0$, and can be easily formulated as an unconstrained or 
constrained
(if the class proportions are required to belong to a simplex) quadratic
program. These authors established asymptotic normality of the estimated
proportions under conditions that are typical of $L^2$ consistency for
kernel density estimates. See \citet{hall03} for additional references on
this strand of work.

Two other works in the machine learning literature have also addressed 
CPE. \citet{saerens01} introduced an EM algorithm in a logistic regression 
framework that adjusts class proportions to maximize the test data 
likelihood given the trained model. \citet{plessis12} developed an 
algorithm based on distribution matching but with a Kullback-Leibler 
criterion. None of the above cited works consider the case where one of 
the classes is unobserved, nor do they establish a consistent 
discrimination rule. Only Hall and Titterington provide theoretical 
analysis for CPE; Hall's analysis considers univariate data, while 
Titterington's assumes the existence of densities.

Multiclass anomaly rejection should not be confused with a problem known
as ``classification with reject option" \citep{chow70}. Despite the name,
that problem is {\em not} concerned with rejection of anomalous instances.
Rather, the classifier is allowed to {\em abstain} from labeling instances
that are ambiguous, that is, near the boundary between two observed
classes. The objective in that problem is to minimize the error rate
conditioned on a label being assigned.

The framework of ``zero-shot learning" can correctly classify previously 
unobserved classes, provided that additional semantic information about 
those classes is also available \citep{palatucci09}. The framework of 
\citet{gornitz13} develops semi-supervised one-class classifiers that 
leverage unlabeled data and are capable of rejecting anomalies, but no 
consistency result is known. In the binary case ($M=2$), MCAR amounts to 
learning with positive and unlabeled examples (LPUE). Consistency for LPUE 
can be established with respect to the Neyman-Pearson criterion 
\citep{blanchard10}, but this analysis has not been extended to other 
performance measures or the multiclass setting. In the next section we 
recount a key contribution of \citet{blanchard10} that enables our own.

\section{Mixture Proportion Estimation}

We will show that class proportion estimation reduces to mixture proportion estimation, which is now reviewed.
Let $(\sX, \mathfrak{S})$ be a measurable space, and
let $F$, $G$, and $H$ be distributions on $\sX$ such that
\begin{equation}
\label{eqn:mpe}
F = (1-\nu) G + \nu H
\end{equation}
where $0 \le \nu \le 1$. Mixture proportion estimation is the following
problem:
given iid training samples
of
sizes $m$ and $n$ from $F$ and $H$ respectively, and no information
about $G$, estimate $\nu$. This problem was first addressed in a
distribution-free framework by \citet{blanchard10} and later applied to the problem of classification
with label noise \citep{scott13}. In this section, we relate the
necessary results from \citet{blanchard10} while following the notation of \citet{scott13}.

Without additional assumptions, $\nu$ is not an identifiable parameter.
Indeed, if $F = (1-\nu) G +\nu H$\, holds, then any alternate
decomposition of the form $F = (1-\nu+\delta) G' + (\nu-\delta) H $\,,
with $G' = (1-\nu+\delta)^{-1}((1-\nu) G + \delta H)$\,, and $\delta \in
[0,\nu)$\,, is also valid. With no knowledge of $G$\,, we cannot decide
which representation is the correct one. Therefore, the idea is to impose
a condition on $G$ such that $\nu$ becomes identifiable. Toward this end,
the following definition is introduced.

\begin{defn}
Let $G$\,, $H$ be probability distributions. $G$ is said to be {\em
irreducible} with respect to $H$ if there exists no decomposition of the
form $G = \gamma H + (1-\gamma) F' $, where $F'$ is some probability
distribution and $0< \gamma \leq 1$\,.
\end{defn}

Some commentary on this definition is offered below. The
following was established in \citet{blanchard10}.

\begin{prop}
\label{prop:canondecmp}
Let $F$\,, $H$ be probability distributions.
If $F \neq H$, there is a unique $\nu^*\in[0,1)$ and $G$ such that the
decomposition $F = (1-\nu^* ) G+ \nu^* H$ holds, and such that $G$ is
irreducible with respect to $H$\,. If we additionally
define $\nu^*=1$ when $F = H$, then in all cases,
\begin{align*}
\nu^* := \max\{\alpha \in[0,1]: &
\, \exists \text{ a distribution $G'$ s.t. } \\
 & F = (1-\alpha)G' + \alpha H \}\,.
\end{align*}
\end{prop}

By this result, the following is well-defined.
\begin{defn}
For any two probability distributions $F$, $H$, define
\begin{align*}
\nu^*(F,H) := \max\{\alpha \in[0,1]: &
\, \exists \text{ a distribution $G'$ s.t. } \\
 & F = (1-\alpha)G' + \alpha H \}\,.
\end{align*}
\end{defn}

Thus, $G$ is irreducible with respect to $H$ if and only if $\nu^*(G, H) =
0$. Further, it is not hard to show that for any two distributions $F$ and
$H$, $\nu^*(F,H) = \inf_{A \in \mathfrak{S}} F(A)/H(A)$ \citep{scott13}.
Similarly, when $F$ and $H$ have densities $f$ and $h$, $\nu^*(F,H)$ is
the essential infimum of $f(x)/h(x)$. These identities make it possible to
check irreducibility in different scenarios. For example, $\nu^*(G,H)=0$
whenever the support of $G$ does not contain the support of $H$. Even if
the supports are equal, irreducibility can still hold as in the case where
$g$ and $h$ are two Gaussian densities with distinct means, where the
variance of $h$ is no smaller than the variance of $g$ \citep{scott13}.

The following corollary summarizes the above and states
that irreducibility of $G$ w.r.t. $H$ is a sufficient condition
for $\nu$ in \eqref{eqn:mpe} to be identifiable.
\begin{cor}
\label{cor:irrd}
If $F = (1-\gamma) G + \gamma H$, and $G$ is irreducible with respect to
$H$, then $\gamma = \nu^*(F,H)$.
\end{cor}

\citet{blanchard10} studied an estimator $\widehat{\nu} =
\widehat{\nu}(\widehat{F},\widehat{H})$ of $\nu^*(F,H)$, where
$\widehat{F}$ and $\widehat{H}$ denote the empirical distributions based
on iid random samples from $F$ and $H$. They show in Thm. 8 that
$\widehat{\nu}$ is strongly universally consistent, i.e., for any $F$ and
$H$, $\widehat{\nu} \to \nu^*(F,H)$ in probability as the 
sample sized tend to $\infty$.\footnote{More
precisely, \citet{blanchard10} use the notation $\pi = 1-\nu$, and present
a consistent estimator for $\pi$. Furthermore, they actually establish 
almost sure convergence. As noted by \citet{scott13}, the statement of 
Thm. 8 of \citet{blanchard10} needs to be amended slightly (by 
constraining how the two sample sizes grow w.r.t. each other) for almost 
sure convergence to hold.} We will show that this estimator 
leads to consistent estimators of class probabilities. The estimator is
discussed further in Sec. \ref{sec:mpealg}.

\section{Class Proportion Estimation}
\label{sec:cpe}

In this section we apply mixture proportion estimation to CPE. Let $P_1,
\ldots, P_M$ be probability measures (distributions) on $(\sX,
\mathfrak{S})$.

\subsection{Identifiability Conditions}

As with mixture proportion estimation, class proportion estimation
requires an identifiability condition.
\begin{description}
\item[(A)] For all $i = 1, \ldots, M$, every element of $\conv
\{P_\ell \, :
\, \ell \ne i\}$ is irreducible with respect to $P_i$.
\end{description}
Here $\conv\{Q_1, \dots, Q_K\}$ denotes the set of convex combinations of
$Q_1, \ldots, Q_K$, that is, the set of mixture distributions based on
$Q_1, \ldots, Q_K$.
To illuminate {\bf (A)}, we introduce a second condition, where
$\supp(Q)$ denotes the support of distribution $Q$.
\begin{description}
\item[(B)] For all $i=1,\ldots, M$, $\supp(P_i) \nsubseteq
\cup_{\ell \ne i}
\supp(P_\ell)$.
\end{description}
{\bf (B)} clearly implies {\bf (A)} from the definition of
irreducible.

We argue that {\bf (B)} is a reasonable assumption in many real-world
classification problems, and therefore so is {\bf (A)}. In words, {\bf
(B)} means that for each class, there exist at least some instances, with
positive probability of occurring (however small), that are always
correctly classified by an optimal classifier. In other words, such
instances could not possibly be mistaken for instances of another class.
For example, consider handwritten digit recognition. Although various
classes may have overlapping supports, each class has instances
(corresponding to very clear handwriting, say) that could not possibly be
mistaken for any other class.

\subsection{Consistency in the Fully Observed Case}
\label{sec:full}

For now assume training samples from all $M$ classes are observed. Under {\bf (A)}, the proportions $\pi_i$ are identifiable, and we propose to estimate them via
\begin{equation}
\label{eqn:piest}
\pihat_i = \nuhat(\Ptesthat, \Phat_i)
\end{equation}
for $i=1, \ldots, M$, where $\nuhat$ is the estimator of \citet{blanchard10} discussed in the previous section.

\begin{prop}
\label{prop:cpe}
Under {\bf (A)}, for each $i$, $\pihat_i$ converges to $\pi_i$ in 
probability as $\min\{n_0, n_i\} \to \infty$.
\end{prop}
\begin{proof}
WLOG assume $i=1$. Now $\Ptest = \pi_1 P_1
+ (1-\pi_1)Q$ where $Q \in \conv\{P_\ell : \ell \ne 1\}$. Under
{\bf (A)}, $\nu^*(Q,P_1)
= 0$, and therefore by Corollary \ref{cor:irrd}, $\pi_1 = \nu^*(\Ptest,
P_1)$. The result now follows by convergence in probability of
$\nuhat(\Ptesthat, \Phat_1)$ to $\nu^*(\Ptest, P_1)$.
\end{proof}

When $M=2$, {\bf (A)} says
$\nu^*(P_1,P_2) = 0$ and $\nu^*(P_2,P_1) =0$. This is the so-called {\em
mutual irreducibility} assumption adopted by \citet{scott13} in the context
of label noise. It turns out that when $M=2$ we can consistently estimate the
proportions under a weaker condition, namely,
$P_1 \ne P_2$. To achieve this, we employ the following estimators:
$$
\pihat_1' := \frac{1 - \nuhat(\Phat_0,\Phat_2)}{1 -
\nuhat(\Phat_1,\Phat_2)}, \ \ \ \
\pihat_2' := \frac{1 - \nuhat(\Phat_0,\Phat_1)}{1 -
\nuhat(\Phat_2,\Phat_1)}.
$$
The intuition is that in the binary case, even if {\bf (A)} is violated,
say $\nu^*(P_1, P_2) > 0$, we can use mixture proportion estimation to
estimate $\nu^*(P_1, P_2)$, and rescale the estimates
accordingly.
Note that each of these modified estimators uses all three
samples, and therefore this result does not generalize to the case where
one class is unobserved.

\begin{prop}
\label{prop:binary}
If $M=2$ and $P_1 \ne P_2$, then $\pihat_1' \to \pi_1$ in probability and
$\pihat_2' \to \pi_2$ in probability, as $\min\{\ntest, n_1,
n_2\} \to \infty$.
\end{prop}
\begin{proof}
Consider estimation of $\pi_1$. Denote $\nu_{12} = \nu^*(P_1,P_2)$. By
Proposition \ref{prop:canondecmp}, there exists a unique distribution
$E_1$ such that $P_1 = (1-\nu_{12}) E_1 + \nu_{12} P_2$ and $\nu(E_1, P_2)
= 0$. Then
\begin{align*}
\Ptest &= \pi_1 [(1-\nu_{12}) E_1 + \nu_{12} P_2] + (1-\pi_1) P_2 \\
&= \pi_1 (1-\nu_{12}) E_1 + [\pi_1 \nu_{12} + (1-\pi_1)] P_2.
\end{align*}
Since $\nu(E_1, P_2) = 0$, by Corollary \ref{cor:irrd} we must have
$\nu^*(\Ptest, P_2) = \pi_1 \nu_{12} + (1-\pi_1)$. Solving for $\pi_1$
yields $\pi_1 = \frac{1-\nu^*(\Ptest, P_2)}{1-\nu^*(P_1, P_2)}$.
Since $P_1 \ne P_2$, the denominator is nonzero. The result now follows by
consistency of $\nuhat$ and continuity of division.
\end{proof}


\subsection{Consistent CPE with an Unobserved Class}
\label{sec:partial}

The primary advantage of our approach to CPE is that it can consistently 
estimate all proportions, even $\pi_M$, when $n_M=0$. The estimators 
$\pihat_i$ of Eqn. \eqref{eqn:piest} do not depend on $\widehat{P}_M$ when 
$i < M$, so they can remain the same in this setting. For 
$i=M$, we can just set $\pihat_M := 1 - \sum_{i=1}^{M-1} \pihat_i$.
The following is an immediate consequence of the necessary condition 
$\sum_{i=1}^M \pi = 1$ and the consistency of $\pihat_1, \ldots, 
\pihat_{M-1}$.
\begin{cor}
Consider class proportion estimation where $n_M = 0$. Let $\pihat_i$ be as 
in Eqn. \eqref{eqn:piest} for $i = 1, \ldots, M-1$, and set  $\pihat_M = 1 
-\sum_{i=1}^{M-1} \pihat_i$. Under {\bf (A)}, for each $i = 1, 
\ldots, M$, $\pihat_i$ converges to
$\pi_i$ in probability as $\min\{n_0, n_1, \ldots, n_{M-1}\} \to \infty$.
\end{cor}

\section{Anomaly Rejection}

We now turn our attention to the design of a consistent discrimination
rule for MCAR. In this setting, available data consist of iid random
samples from $P_1, \ldots, P_{M-1}$ as in \eqref{eqn:traindata}, and an
iid random sample from $P_0$ as in \eqref{eqn:testdata}. Data from $P_M$
are not observed. Our goal is a discrimination rule $\widehat{f}$,
constructed from the available data, whose risk converges to the Bayes
risk as the various sample sizes tend to $\infty$. Note that previous work
has not addressed this problem even in the case where all classes are
observed (which still differs from standard classification 
because the test distribution has different class proportions).

To set notation, let $Q$ denote the joint distribution of $(X,Y) \in \sX
\times \{1, \ldots, M\}$ such that the $X$-marginal of $Q$ is $P_0$, the
$Y$-marginal is given by the $\pi_i$, and the class-conditional
distributions are $P_i$. For any classifier $f:\sX \to \{1, \ldots, M\}$,
denote the class-conditional error probabilities
$R_{i}(f) := P_i(\{x: f(x) \ne i\})$,
and  the test-distribution risk $R(f) := Q(\{(x,y):
f(x) \ne y\}) = \sum_{i=1}^M \pi_i R_i(f)$. Let $R^*$ denote the Bayes
risk for distribution $Q$. Our goal is to construct a discrimination rule
$\widehat{f}$ such that $R(\widehat{f}) \to R^*$ in probability as the
sample sizes $n_0, n_1, \ldots, n_{M-1}$ tend to $\infty$.

To construct such a rule, we adapt a classic strategy from statistical
learning theory \citep{devroye96}: empirical risk minimization (ERM) over a
growing family of classifiers, also known as sieve estimation. This strategy relies upon VC theory, and
since we are in a multiclass setting, we take the following generalization
of VC dimension to multiclass. Define the (multiclass) VC dimension of a
set of classifiers $\sF$ to be the maximum conventional (two-class) VC
dimension \citep{devroye96} of the family of sets $\{x : f(x) \ne \ell\}_{f
\in \sF}$, over $\ell = 1, \ldots, M$.

As its name suggests, ERM also requires an estimate of the risk. We
propose to estimate $R(f)$ by writing $R(f) = \sum_{i=1}^{M-1} \pi_i
R_i(f) + \RMul(f)$, where $\RMul(f) := Q(\{(x,y) : f(x) \ne y, y
= M\}) = \pi_M R_M(f)$, and estimating each term in this expression.
For $i <M$, $\pi_i$ is estimated by $\pihat_i$ in Eqn.
\eqref{eqn:piest}, and $R_i(f)$ is estimated by
$\Rhat_i(f) := \frac1{n_i} \sum_{j=1}^{n_i} \ind{f(X_j^i) \ne i}$.
An estimate of $\RMul(f)$ is motivated as follows. Let $R_{i M}(f) :=
P_i(\{x: f(x) \ne M\})$ and observe that $R_{0M}(f) =
\sum_{i=1}^{M-1} \pi_i R_{iM}(f) + \pi_M R_M(f)$. Then
\begin{equation}
\label{eqn:RMul}
\mbox{$\RMul(f) = R_{0M}(f) - \sum_{i=1}^{M-1} \pi_i R_{iM}(f)$}.
\end{equation}
Plugging in $\Rhat_{i M}(f) := \frac1{n_i}
\sum_{j=1}^{n_i} \ind{f(X_j^i) \ne M}$ and our estimates for the
$\pi_i$ leads to the following estimator:
\begin{equation}
\label{eqn:RMhatul}
\mbox{$\RMhatul(f) = \Rhat_{0M}(f) -
\sum_{i=1}^{M-1} \pihat_i \Rhat_{iM}(f)$}.
\end{equation}
Now set $\Rhat(f) :=
\sum_{i=1}^{M-1} \pihat_i \Rhat_i(f) + \RMhatul(f)$.

We now define the ERM-based discrimination rule.
Let $( \sF_k )_{k\ge1}$ be a sequence of VC
classes with corresponding (multiclass) VC dimensions $V_k < \infty$. Let
$\tau_k$ be any sequence of positive numbers tending to zero.
Let $\fhat_k$ be an approximate empirical risk minimizer, i.e., any 
classifier
$$
\fhat_k \in \Big\{f \in \sF_k \, : \, \Rhat(f) \le \inf_{f' \in
\sF_k} \Rhat(f') + \tau_k\Big\}.
$$
The introduction of $\tau_k$ lets us avoid assuming the existence of an
empirical risk minimizer.
Denote $\nb := (n_0, n_1, \ldots, n_{M-1})$. We write $\nb \to
\infty$ to indicate $\min \{n_0, n_1, \ldots, n_{M-1}\} \to \infty$. Let
$k(\nb)$ denote a sequence of positive integers indexed by $\nb$. Finally,
define the discrimination rule $\fhat := \fhat_{k(\nb)}$. Note that the
sequences $(\sF_k)_{k\ge 1}$ and $k(\nb)$ are user-specified and must grow
in a certain way, indicated by the theory below, for $\fhat$ to be
consistent.

Analysis of this discrimination rule hinges on
uniform control of the deviation $|R(f) - \Rhat(f)|$ over $\sF_{k(\nb)}$
as $\nb \to \infty$. The following result establishes this property. In
the proof, the error deviance is decomposed in such a way that uniform
control follows from the multiclass VC extension and consistency of the
class proportion estimators. The proof of this and the next result are
found in the supplemental material.

\begin{prop}
Assume {\bf (A)} holds and suppose $k(\nb) \to \infty$ as $\nb \to \infty$
such that
\begin{equation}
\label{eqn:vccond}
\frac{V_{k(\nb)} \log n_i}{n_i} \to 0,
\end{equation}
for $0 \le i \le M-1$. Then
$$
\sup_{f \in \sF_{k(\nb)}} |R(f) - \Rhat(f)| \to 0
$$
in probability as $\nb \to \infty$.
\end{prop}

So that arbitrary classifiers can be accurately approximated,
we choose $(\sF_k)_{k\ge 1}$ satisfying the following universal
approximation property: For any joint distribution $Q$ on $\sX \times
\{1, \ldots, M\}$,
$$
\lim_{k \to \infty} \inf_{f \in \sF_k} R(f) = R^*
$$
where $R^*$ is the Bayes error corresponding to $Q$.
\citet{devroye96} give examples of families of VC
classes that satisfy the
above approximation property. We can now state the main result of this
section.
\begin{thm}
Assume {\bf (A)} holds and that $(\sF_k)_{k\ge 1}$ is chosen to satisfy
the universal approximation property above. Further suppose $k(\nb)$ is
chosen such that as $\nb \to \infty$,
$k(\nb) \to \infty$ and \eqref{eqn:vccond} holds for $0 \le i \le M-1$.
Then $R(\fhat) \to R^*$ in probability.
\end{thm}
Although we have focused on the probability of error as a performance
measure, it would not be difficult to adapt this result to any other
performance measure that is a continuous function of the class proportions
$\pi_i$ and class-conditional errors $R_i$, such as a cost-sensitive Bayes
risk or the minmax error.

\section{Implementation and Experiments}
\label{sec:impexp}

In this section we introduce a practical algorithm for mixture proportion estimation (MPE) and use it to
implement the proposed CPE methodology. We then compare our method to
existing methods for CPE on a variety of binary and multiclass data sets.
We consider two experimental settings. In the first setting, we adopt the assumption that the unlabeled
test data do not contain an anomalous class. This is the assumption adopted by competing methods and, not surprisingly, we find that
they outperform our own approach, which allows for the existence of an
anomalous class in the test data. In the second group of experiments, the test
data contain an anomalous class, and our approach vastly outperforms the
competitors in this scenario.

For a fairer head-to-head comparison with existing methods, we introduce
two additional class proportion estimators based on MPE that make the same
assumptions as competing methods (namely, that there is not an anomalous class
in the test data). We compare these to existing methods under the first
experimental setting and find they are competitive, which offers
experimental validation of the MPE-based framework.

A thorough experimental investigation of MCAR is beyond the scope of this
work. The discrimination rule we introduce for MCAR could be implemented
for various VC classes such as histograms or decision trees, but other methods
would also be worthy of exploration, such as those based on convex
surrogate losses.

\subsection{Practical Algorithm for MPE}
\label{sec:mpealg}

As discussed in \citet{scott13}, Theorem 6 of \citet{blanchard10} tells us $\nu^* = \nu^*(F,H)$ is related to the
optimal Receiver Operating Characteristic (ROC) that arises when the
distribution
$H$ is viewed as the null hypothesis and $F$ as the alternative.
This optimal ROC is the function\footnote{Technically, if
the function is not concave, the optimal ROC is the smallest concave
function that upper bounds $p(\alpha)$.}
\begin{align*}
p(\alpha):= \sup_{C \subset \sX} & \, F(C) \\
\mbox{s.t.} & \, H(C) \le \alpha.
\end{align*}
This function gives the optimal detection probability of a binary
classifier constrained to have false alarm rate no more than $\alpha$,
where $C$ here represents a subset of $\sX$ that predicts the class of
$F$.

As shown in \citet{blanchard10,scott13}, $\nu^* = \left.
\frac{dp}{d\alpha}\right|_{\alpha = 1^-}$, the slope of the optimal ROC
evaluated at the right endpoint where the false positive rate becomes 1.
The estimator $\nuhat$ studied in \citet{blanchard10}
implements this principle, but relies on distribution free confidence
intervals (to achieve universal consistency), and thus tends to be too
conservative in practice.



Therefore we introduce a more practical implementation of the above
principle for MPE, and apply it to CPE. Given random samples $\widehat{F}$
and $\widehat{H}$ from $F$ and
$H$, we treat these as training classes for a binary classification
problem, and train a kernel logistic regression (KLR) classifier using a
Gaussian kernel. We then vary the threshold on the KLR posterior class
probability to generate an empirical version of the optimal ROC, and
obtain $\nuhat$ by estimating the slope of this empirical ROC at its right
endpoint. Note that the choice to use KLR is simply for convenience, and any
binary classifier capable of producing an ROC, such as cost-sensitive SVMs, could be used instead.

Since the empirical ROC may be noisy at its right endpoint, we fit a curve
to the empirical ROC and take the right endpoint slope of the fitted curve
to be our proportion estimate. \citet{lloyd00} provides two
regression models for ROCs, and we augment them both to include an extra
linear term in an attempt to better model the linear behavior seen towards
the right end of the ROC.

In particular, for a given ROC, let $\alpha$ denote the false positive
rate, $p(\alpha)$ the corresponding detection rate, and $f(\alpha)$ the model for $p(\alpha)$.
Our regression models are:
\begin{equation}
\label{eqn:gaussianregression}
f_{\gamma, \Delta}(\alpha) = (1-\gamma)Q(Q^{-1}(\alpha) + \Delta)
+ \gamma \alpha.
\end{equation}
\begin{equation}
\label{eqn:concaveregression}
f_{\gamma, \Delta, \mu}(\alpha) =
(1-\gamma)(1 + \Delta(\alpha^{-\mu} - 1))^{-\frac{1}{\mu}} + \gamma
\alpha.
\end{equation}
where $Q$ is the standard normal CDF, $\Delta$ controls
ROC quality, $\mu$ is an asymmetry parameter, and $\gamma$ is the slope of
the added linear component. See \citet{lloyd00} for more insight into the form
of these models.

Since the domain and range of the ROC are probabilities, we fit the
models by minimizing the binomial deviance between the empirical ROC
given by $\pfhat_j$ and $\pdhat_j$, where $j = 1, \ldots, n$ indexes
sample
points along the empirical ROC, and the model $f(\pfhat)$ as given by
Eqns. \eqref{eqn:gaussianregression} or \eqref{eqn:concaveregression}:
$$
B_f(\pfhat, \pdhat) = -2 \sum_{j=1}^{n} \pdhat_j \log(f(\pfhat_j))
+ (1-\pdhat_j) \log(1-f(\pfhat_j))
$$
The right-endpoint slope of the model as a function of the fitted
parameters is $\gamma$
in the case of \eqref{eqn:gaussianregression} and $(1-\gamma) \Delta +
\gamma$ in the case of \eqref{eqn:concaveregression}.

\subsection{New MPE-based Algorithms for CPE}
\label{sec:imp}

We apply the above algorithm to CPE following the framework of Sec. 
\ref{sec:cpe}, so that $\pihat_i := \nuhat(\widehat{P}_0, \widehat{P}_i)$,
where recall $\widehat{P}_0$ and $\widehat{P}_i$ represent the data drawn from
the unlabeled test distribution and training class $i$ respectively. In the first set of experiments, there are 
$M$ observed training classes, and our method allows for the existence of 
an $(M+1)st$ class, estimating $\pihat_{M+1} = 1 - \sum_{i=1}^{M} 
\pihat_i$. In the second set of experiments, there are $M-1$ training 
classes, and the anomalous class proportion $\pi_M$ is estimated as 
$\pihat_M = 1 - \sum_{i=1}^{M-1} \pihat_i$. We found the model from Eqn. 
\eqref{eqn:concaveregression} performed best. In the results we denote 
this CPE method as \emph{MPE-Incomplete} since it assumes incomplete 
knowledge of the classes.





In the fully observed case (the first experimental setting), we showed in
Sec. \ref{sec:full} that our approach consistently estimates the true
class proportions. However, due to estimation error the estimates $\widehat{\pi}_1, \ldots,
\widehat{\pi}_M$ do not sum to one, as they should in this setting.
Therefore, for a fairer comparison with existing methods, we also
introduce two extensions of MPE-based CPE that, like previous methods, do
not support an anomalous class in the test data, but do perform better
when all classes are observed.

The first extension is to simply project the vector of estimated
proportions onto the probability simplex $\Delta^M$. In the results, we
denote this projected estimate as \emph{MPE-Projected}.

The second extension forms $M$ empirical ROCs based on the distributions
$(\Ptest,P_i)$, $i=1, \ldots, M$, and fits all ROC
curves simultaneously while constraining the estimated class proportions to
sum to one. We use the model from Eqn. \eqref{eqn:gaussianregression}
since the slope at the right endpoint is simply $\gamma$. Letting
$f$ be Eqn. \eqref{eqn:gaussianregression}, and $B_f$ the binomial
deviance given above, we solve
\begin{equation*}
\begin{aligned}
& \underset{\gamma_i, \Delta_i}{\text{minimize}}
& & \sum_{i=1}^M B_f(\pfhat^i, \pdhat^i)
\text{,  subject to} \sum_{i=1}^M \gamma_i = 1
\end{aligned}
\end{equation*}
where $(\pfhat^i,\pdhat^i)$ is the empirical ROC based on $\widehat{P}_0$ 
and $\widehat{P}_i$.
This extension is denoted \emph{MPE-Joint}.

\subsection{Evaluation}
\label{sec:exp}


Recall that we consider two experimental settings. In the first, all
training classes are observed, while in the second, the $M$th class is not
observed.

We compare against several approaches noted in the related work section. 
We denote the methods by \citet{saerens01}, \citet{titterington83}, and 
\citet{plessis12} as EM, $L^2$ Distance, and KL-Divergence\footnote{Due to 
computational constraints, we limited the input to the KL-Divergence 
method to 1000 training and 1000 testing examples, and were not able to 
use it in the multiclass setting.}, respectively. Since the EM algorithm 
requires posterior class probabilities, we use kernel logistic regression 
in both the EM algorithm and our method. Finally, we compare against a 
simple baseline estimate defined as the proportions of the labels 
predicted by a KLR classifier on the test data.

Our experiments were conducted on 13 well-known binary data sets and 5
multiclass data sets. Each data set was permuted 10 times and performance
was computed by averaging over permutations. To measure performance we use
the $\ell_1$-norm between the estimated class proportion vector and the
vector of true class proportions. For each data set and permutation, we
manually set the class proportion of the $M$th class to range over the
following set of values: \{1\%, 10\%, 20\%, $\ldots$, 90\%, 99\%\}. In the
binary case, the positive class proportion was taken to be the $M$th class
($M=2$). In the multiclass case, the largest class in the original data
set
was taken to be the $M$th class. The size of both the training set and
testing set were kept constant over all proportions. As a result, as the
$M$-th class grows the remaining classes shrink proportionately.

In the first experimental setting, the $M$th class is observed. Under the assumption that
all classes are observed, and to fairly compare to the other methods, in this scenario we discard the
estimate of the $(M+1)st$ class proportion for the \emph{MPE-Incomplete} method. Table
\ref{table:performance} reports the $\ell_1$-norm performance measure
means and standard deviations, where the average is taken over permutation
and varied class proportion. Fig. \ref{fig:BinaryProportionsError} shows
the performance of each method, averaged over the binary data sets, as a
function of the artificially modified class proportion.

\begin{table*}[ht]
\centering
\caption{Comparison of mean performances with standard deviations, taken
over all data permutations and resampled proportions.}
\small
\tabcolsep 3.0pt
\resizebox{\columnwidth}{!}{%
\begin{tabular}{c| c c c c c c c}
\hline\hline
Data set (M) & MPE-Incomplete & MPE-Projected & MPE-Joint & EM-KLR & $L^2$ Dist. & KL-Diverg. & baseline \\
\hline
All Binary & .188 $\pm$ .20 & .131 $\pm$ .17 & .140 $\pm$ .20 & .145 $\pm$ .21 & \bf .104 $\pm$ .12 & .155 $\pm$ .17 & .270 $\pm$ .39\\
All Multiclass & .143 $\pm$ .08 & .137 $\pm$ .09 & .114 $\pm$ .07 & .098 $\pm$ .14 & .109 $\pm$ .08 & n/a & \bf .097 $\pm$ .10\\
\hline
Australian (2) & .169 $\pm$ .12 & .132 $\pm$ .13 & .094 $\pm$ .07 & .096 $\pm$ .08 & \bf .077 $\pm$ .06 & .164 $\pm$ .14 & .179 $\pm$ .12\\
Banana (2) & .045 $\pm$ .04 & .030 $\pm$ .04 & .019 $\pm$ .02 & \bf .016 $\pm$ .02 & .128 $\pm$ .08 & .296 $\pm$ .22 & .117 $\pm$ .07\\
Breast-cancer (2) & .535 $\pm$ .20 & .312 $\pm$ .24 & .488 $\pm$ .32 & .442 $\pm$ .35 & \bf .234 $\pm$ .17 & .235 $\pm$ .19 & .875 $\pm$ .58\\
Diabetes (2) & .221 $\pm$ .10 & .152 $\pm$ .11 & .201 $\pm$ .17 & .133 $\pm$ .12 & \bf .112 $\pm$ .09 & .182 $\pm$ .18 & .393 $\pm$ .29\\
German (2) & .307 $\pm$ .15 & .188 $\pm$ .17 & .219 $\pm$ .18 & .211 $\pm$ .17 & \bf .146 $\pm$ .10 & .180 $\pm$ .13 & .645 $\pm$ .47\\
Image (2) & .086 $\pm$ .06 & .066 $\pm$ .06 & .044 $\pm$ .04 & \bf .020 $\pm$ .02 & .083 $\pm$ .07 & .134 $\pm$ .11 & .053 $\pm$ .04\\
Ionosphere (2) & .217 $\pm$ .17 & .176 $\pm$ .17 & .129 $\pm$ .11 & \bf .052 $\pm$ .04 & .125 $\pm$ .10 & .140 $\pm$ .12 & .098 $\pm$ .08\\
Ringnorm (2) & .023 $\pm$ .03 & .018 $\pm$ .03 & \bf .010 $\pm$ .01 & .165 $\pm$ .20 & .014 $\pm$ .01 & .022 $\pm$ .01 & .018 $\pm$ .01\\
Saheart (2) & .406 $\pm$ .20 & .283 $\pm$ .22 & .364 $\pm$ .27 & .222 $\pm$ .19 & \bf .184 $\pm$ .15 & .225 $\pm$ .18 & .552 $\pm$ .39\\
Splice (2) & .088 $\pm$ .07 & .073 $\pm$ .07 & \bf .049 $\pm$ .05 & .050 $\pm$ .03 & .050 $\pm$ .04 & .080 $\pm$ .06 & .105 $\pm$ .06\\
Thyroid (2) & .265 $\pm$ .19 & .204 $\pm$ .20 & \bf .153 $\pm$ .13 & .183 $\pm$ .28 & .163 $\pm$ .17 & .300 $\pm$ .25 & .339 $\pm$ .54\\
Twonorm (2) & .022 $\pm$ .02 & .018 $\pm$ .01 & \bf .010 $\pm$ .01 & .269 $\pm$ .21 & .010 $\pm$ .01 & .023 $\pm$ .01 & .025 $\pm$ .01\\
Waveform (2) & .063 $\pm$ .04 & .045 $\pm$ .03 & .043 $\pm$ .03 & .028 $\pm$ .02 & \bf .019 $\pm$ .02 & .036 $\pm$ .03 & .113 $\pm$ .07\\
\hline
SensIT (3) & .189 $\pm$ .08 & .140 $\pm$ .09 & .169 $\pm$ .08 & .340 $\pm$ .16 & \bf .104 $\pm$ .06 & n/a & .210 $\pm$ .12\\
DNA (3) & .080 $\pm$ .04 & .074 $\pm$ .04 & .048 $\pm$ .03 & \bf .025 $\pm$ .02 & .062 $\pm$ .03 & n/a & .055 $\pm$ .02\\
Opportunity (4) & .154 $\pm$ .07 & .158 $\pm$ .08 & .116 $\pm$ .05 & \bf .067 $\pm$ .04 & .156 $\pm$ .14 & n/a & .136 $\pm$ .09\\
SatImage (6) & .109 $\pm$ .06 & .115 $\pm$ .08 & .085 $\pm$ .04 & \bf .031 $\pm$ .01 & .083 $\pm$ .04 & n/a & .059 $\pm$ .02\\
Segment (7) & .183 $\pm$ .08 & .196 $\pm$ .11 & .152 $\pm$ .07 & .027 $\pm$ .01 & .139 $\pm$ .05 & n/a & \bf .025 $\pm$ .02\\
\hline
\end{tabular}
}
\label{table:performance}
\end{table*}

The results show that the \emph{MPE-Projected} and \emph{MPE-Joint} extensions are comparable
to the best performing algorithms in the binary case, and achieve the best
performance on a few data sets.
In some multiclass data sets the baseline error is low indicating the
classes are
highly separable. The EM algorithm often performed well but had high
variance. The $L^2$ Distance method performed consistently well and best
overall. The \emph{MPE-Incomplete} method does not assume the test distribution
contains only training classes, yet, it still performs reasonably well.
Using a Wilcoxon signed rank test, we found the mean performances (across data set
and varied proportion) of the algorithms were significantly different at
the 5\% level,
except the \emph{MPE-Projected}, \emph{MPE-Joint}, and EM methods in
the binary case were mutually insignificant from each other.

In the second experimental setting, the $M$th class is not available to
the various algorithms. Since competing methods do not natively support
this scenario, we allow them to estimate the class proportions of classes
they have observed and set their estimate of the anomalous class
proportion to zero. Predictably, as shown in Fig.
\ref{fig:PartialObserve}, the performances of competing methods (averaged
over data sets) rise linearly as the anomalous class proportion grows. The
\emph{MPE-Incomplete} method, in contrast, adapts to the anomalous class.

\begin{figure}
\centerline{\includegraphics[width=5.0in]{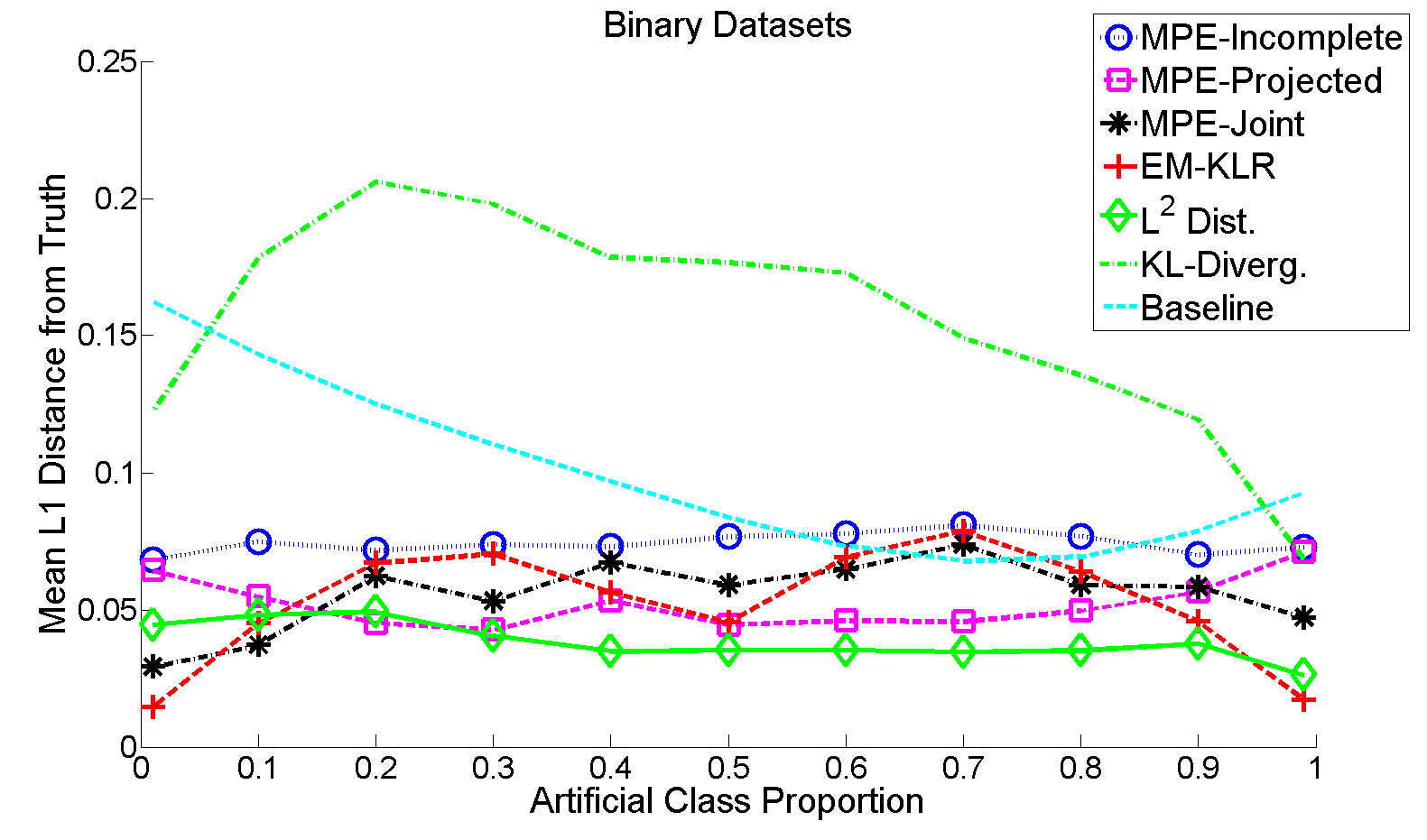}}
\caption{Mean performance over all permutations and binary data sets
as manipulated class proportion changes.}
\label{fig:BinaryProportionsError}
\end{figure}

\begin{figure}
\centerline{\includegraphics[width=5.0in]{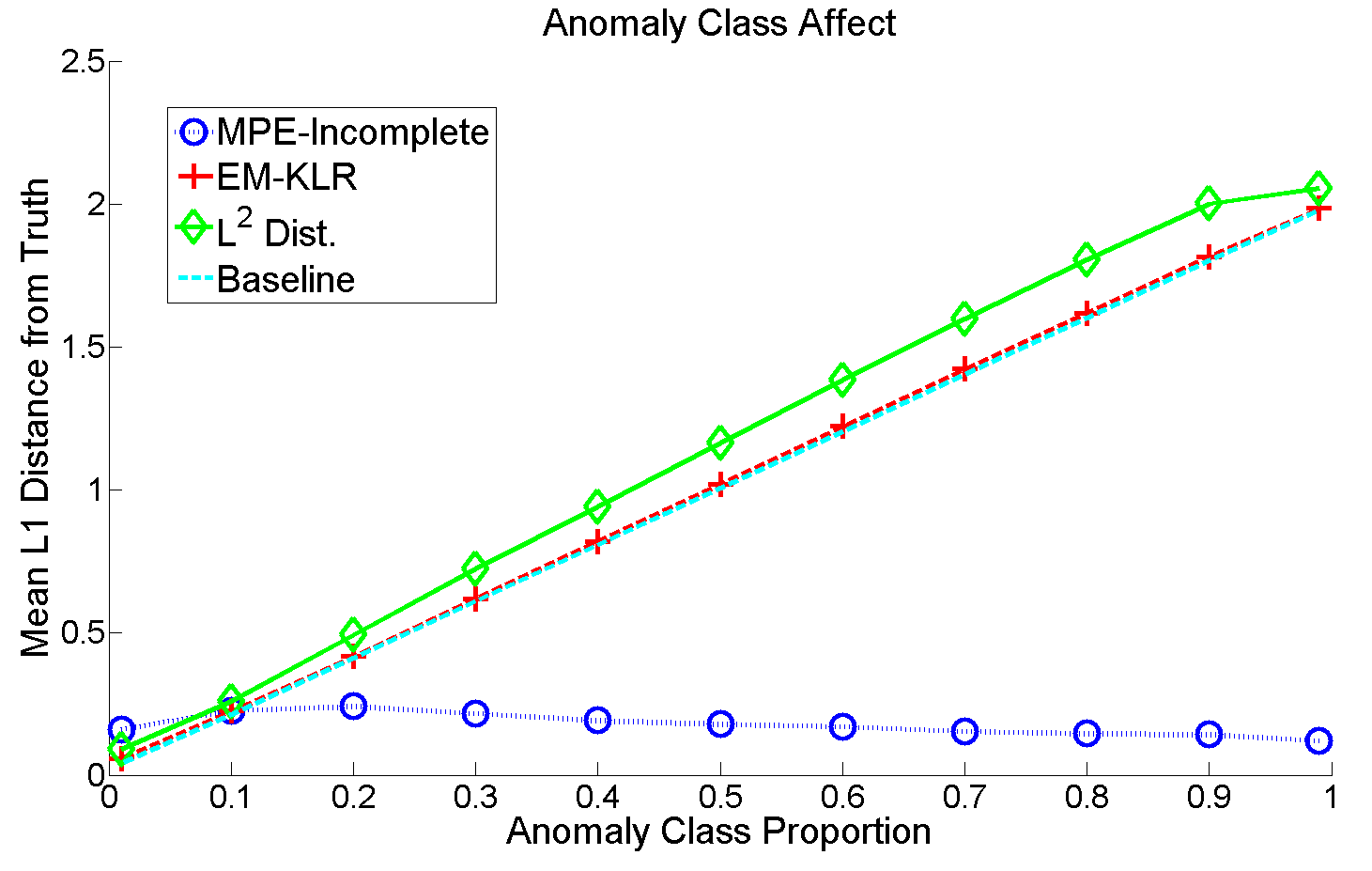}}
\caption{Mean performance over all permutations and multiclass data sets
as anomaly class proportion changes.}
\label{fig:PartialObserve}
\end{figure}

In the supplemental material, additional details of the experiments are
reported. We also describe a method that successfully estimates confidence
intervals on the $\pi_i$, with experimental results.

\section{Conclusion}

This work has demonstrated, both theoretically and experimentally, that
{\em mixture proportion estimation} can be successfully applied to the
problem of class proportion estimation. Unlike existing methods for CPE,
our approach is able to accurately estimate the proportion of an anomalous
class in the unlabeled test data. This feature of our method facilitates
error estimation with respect to the test distribution, which forms the
basis of a consistent discrimination rule for multiclass anomaly
rejection. These approaches based on MPE are, to our knowledge, the first
viable solutions to these two fundamental domain adaptation problems.

\section*{Acknowledgements}
C. Scott was supported in part by NSF Grants 0953135, 1047871, and 1217880.

\section*{Appendix}

\appendix

\renewcommand{\theequation}{S.\arabic{equation}}
\setcounter{equation}{0}

\section{Proof of Proposition 4}
Observe
\begin{align}
\lefteqn{|R(f) - \Rhat(f)| = | \RMul(f) - \RMhatul(f) |} \nonumber \\
&\qquad + \Bigg|\sum_{i=1}^{M-1} \pi_i (R_i(f) - \Rhat_i(f)) + 
\sum_{i=1}^{M-1} (\pi_i - \pihat_i) \Rhat_i(f) \Bigg| \nonumber \\
&\le |\RMul(f) - \RMhatul(f)| \nonumber \\
&\qquad +  \sum_{i=1}^M |R_i(f) - \Rhat_i(f)| + \sum_{i=1}^{M-1} |\pi_i - 
\pihat_i|. 
\label{eqn:riskbound}
\end{align}

From \eqref{eqn:riskbound} and by consistency of the 
$\pihat_i$, it suffices to show that 
\begin{equation}
\label{eqn:unifM}
\sup_{f \in \sF_{k(\nb)}} |\RMul(f) - \RMhatul(f)| \to 0
\end{equation}
and that for each $i$, $1 \le i < M$,
\begin{equation}
\label{eqn:unifi}
\sup_{f \in \sF_{k(\nb)}} |R_i(f) - \Rhat_i(f)| \to 0
\end{equation}
in probability as $\nb \to \infty$.
For $i < M$, \eqref{eqn:unifi} follows from the standard (two-class) 
VC theorem \citep{devroye96}, by \eqref{eqn:vccond}, and because the 
standard VC 
dimension of $\{x : f(x) \ne i\}_{f \in \sF}$ is upper bounded by the
multiclass VC dimension.

To establish \eqref{eqn:unifM}, recall Eqns. \eqref{eqn:RMul} and 
\eqref{eqn:RMhatul}.
For brevity we omit the dependence of $R_{i \ell}$ and $\Rhat_{i \ell}$ on 
$f$ at times. For any $f$
\begin{align*}
\lefteqn{|\RMul(f) - \RMhatul(f)|} \\
& \le \left[ |R_{0M} - \Rhat_{0M}| +
\sum_{i=1}^{M-1} |\pi_i R_{iM} - \pihat_i \Rhat_{iM}| \right] \\
&= \Bigg[ |R_{0M} - \Rhat_{0M}| \\
&\qquad +
\sum_{i=1}^{M-1}  |\pi_i (R_{jM} - \Rhat_{iM}) + (\pi_i - \pihat_i) 
\Rhat_{iM}| \Bigg] \\
&\le \Bigg[ |R_{0M}(f) - \Rhat_{0M}(f)| \\
&\qquad +
\sum_{i=1}^{M-1} \left(  |R_{iM}(f) - \Rhat_{iM}(f)| + |\pi_i - 
\pihat_i| \right) \Bigg].
\end{align*}
Standard VC theory \citep{devroye96} implies that for any $\epsilon > 0$ 
and for $0 \le i \le M-1$, $\sup_{f 
\in \sF_k} |R_{iM}(f) - \Rhat_{iM}(f)| \to 0$ with 
probability one, by \eqref{eqn:vccond}, and because the standard VC 
dimension of $\{x : f(x) \ne M\}_{f \in \sF}$ is upper bounded by the
multiclass VC dimension. The other terms 
tend to zero in probability by consistency of the $\pihat_i$. The 
result 
now follows. 

\section{Proof of Theorem 1}

Consider the decomposition into estimation and approximation errors,
$$
R(\fhat) - R^* = R(\fhat) - \inf_{f \in \sF_{k(\nb)}} R(f) + 
\inf_{f \in \sF_{k(\nb)}} R(f) - R^*.
$$
The approximation error converges to zero by the stated approximation 
property and because $k(\nb) \to \infty$.

To establish convergence in probability of the estimation error, let $\eps 
> 0$. For each positive integer $k$, let $f_k^* \in \sF_k$ such that 
$R(f_k^*) \le \inf_{f \in \sF_{k}} R(f) + \frac{\eps}4$. Then
\begin{align*}
\lefteqn{R(\fhat) - \inf_{f \in \sF_{k(\nb)}} R(f) \le R(\fhat) - 
R(f_{k(\nb)}^*) + \frac{\eps}4} \\
&\le \Rhat(\fhat) - \Rhat(f_{k(\nb)}^*) + \frac{\eps}2 \\
& \ \ \ \ \mbox{(with prob. tending to $1$, by  previous result)} 
\\
&\le \tau_{k(\nb)}  + \frac{\eps}2 \\ 
&\le \eps,
\end{align*}
where the last step holds for $\nb$ sufficiently large. The result now 
follows.

\section{Additional Details of Experiments}

For each permutation of each dataset, hyper-parameters for Kernel Logistic Regression
were selected via grid-search maximizing classification accuracy using 3-fold cross validation.
For the subsequent binary classification step between each training class and the test sample, the
bandwidth parameter from the previous step is used (to save computation) but the regularization
parameter is again selected, this time to maximize area under the ROC curve.

Before fitting our ROC regression models, we employed a Bayesian bootstrap 
method to reduce noise and provide better fits \citep{gu08}. The Bayesian 
bootstrap method also provided confidence intervals on the ROC. By fitting 
the model from Eqn. \eqref{eqn:concaveregression} to the lower 
confidence interval of the ROC, we were able to estimate an upper 
confidence interval on $\pihat$. We estimate a corresponding lower 
confidence interval as one minus the sum of the remaining class upper 
confidence intervals. Table \ref{table:confidenceintervals} shows the 
percentage of true class proportions which fall between the upper and 
lower estimated 95th-percentile confidence intervals. As expected for the 
two sided interval, we see it is valid in greater than 90\% of cases. We 
also find that the bounds are tighter when more examples are available.

\begin{table*}
\caption{Percentage of true class proportions that fall in the estimated $\pihat$ 95th percentile confidence intervals, and the standard deviation of the upper confidence interval from the true class proportion.}
\centering
\begin{tabular}{c| c| c c | c}
\hline\hline
Dataset (\# Classes) & \% in range & Train Counts & Test Counts & Upper-Interval Std. Dev.\\
\hline
All Binary &0.947 & & & 0.26 \\
All Multiclass &0.972 & & & 0.10 \\
\hline
Australian (2) &0.955 & 350 & 153 & 0.17 \\
Banana (2) &0.991 & 2677 & 1188 & 0.06 \\
Breast-cancer (2) &0.900 & 140 & 41 & 0.54 \\
Diabetis (2) &0.991 & 389 & 134 & 0.29 \\
German (2) &0.982 & 506 & 150 & 0.34 \\
Image (2) &0.945 & 1167 & 495 & 0.10 \\
Ionosphere (2) &0.918 & 178 & 63 & 0.23 \\
Ringnorm (2) &0.982 & 3738 & 1832 & 0.03 \\
Saheart (2) &0.891 & 234 & 80 & 0.41 \\
Splice (2) &0.964 & 1605 & 763 & 0.11 \\
Thyroid (2) &0.818 & 109 & 33 & 0.28 \\
Twonorm (2) &0.991 & 3738 & 1849 & 0.03 \\
Waveform (2) &0.982 & 2526 & 824 & 0.08 \\
\hline
SensIT (3) &0.991 & 1011 & 492 & 0.17 \\
DNA (3) &0.985 & 1011 & 474 & 0.09 \\
Opportunity (4) &0.975 & 1150 & 300 & 0.12 \\
SatImage (6) &0.982 & 2241 & 536 & 0.06 \\
Segment (7) &0.949 & 1167 & 165 & 0.09 \\
\hline
\end{tabular}
\label{table:confidenceintervals}
\end{table*}

Note we truncated the sizes of some multiclass datasets in order to process them in a timely manner. Namely, the Opportunity dataset \citep{opportunitydataset}, and the SensIT dataset \citep{sensitdataset}.

\nocite{opportunitydataset}
\nocite{sensitdataset}

\bibliographystyle{unsrtnat}
\bibliography{myBib}

\end{document}